\documentclass[letterpaper]{article} % DO NOT CHANGE THIS
\usepackage[]{aaai25}  % DO NOT CHANGE THIS
\usepackage{times}  % DO NOT CHANGE THIS
\usepackage{helvet}  % DO NOT CHANGE THIS
\usepackage{courier}  % DO NOT CHANGE THIS
\usepackage[hyphens]{url}  % DO NOT CHANGE THIS
\usepackage{graphicx} % DO NOT CHANGE THIS
\usepackage{natbib}  % DO NOT CHANGE THIS AND DO NOT ADD ANY OPTIONS TO IT
\usepackage{caption} % DO NOT CHANGE THIS AND DO NOT ADD ANY OPTIONS TO IT
\usepackage{algorithm}
\usepackage{algorithmic}

\usepackage{newfloat}
\usepackage{listings}
\DeclareCaptionStyle{ruled}{labelfont=normalfont,labelsep=colon,strut=off} % DO NOT CHANGE THIS
\floatstyle{ruled}
\newfloat{listing}{tb}{lst}{}
\floatname{listing}{Listing}
\title{Lifted Successor Generation in Numeric Planning (Extended Version)}
\author{
    Dominik Drexler\textsuperscript{\rm 1}
}
\affiliations{
    Link{\"o}ping University, Sweden\\
    \textsuperscript{\rm 1}dominik.drexler@liu.se
}

\usepackage{booktabs}
\usepackage{multirow}
\usepackage{amssymb}
\usepackage{amsmath}

\usepackage{amsthm}
\usepackage{thmtools, thm-restate}  % must be loaded before \newtheorem
\newtheorem{definition}{Definition}
\newtheorem{theorem}[definition]{Theorem}

\newtheorem{lemma}[definition]{Lemma}

\newcommand{\Omit}[1]{}

\newcommand{\defined}[1]{\emph{#1}}

\newcommand{\tup}[1]{\ensuremath{\langle #1 \rangle}}

\newcommand{\inlinecite}[1]{\citet{#1}}
\newcommand{\egcite}[1]{\citep[e.g.,][]{#1}}

\newcommand{\naturals}{\ensuremath{\mathbb{N}}}
\newcommand{\reals}{\ensuremath{\mathbb{R}}}
\newcommand{\extreals}{\ensuremath{\overline{\reals}}}
\newcommand{\pluseq}{\mathrel{+}=}
\newcommand{\minuseq}{\mathrel{-}=}
\newcommand{\timeseq}{\mathrel{\times}=}
\newcommand{\diveq}{\mathrel{\div}=}
\newcommand{\hull}{\ensuremath{\mathit{hull}}}

\newcommand{\intervals}{\ensuremath{\mathcal{J}}}
\newcommand{\interval}{\ensuremath{J}}

\newcommand{\predicates}{\ensuremath{\mathcal{P}}}
\newcommand{\predicate}{\ensuremath{P}}
\newcommand{\atom}{\ensuremath{p}}
\newcommand{\groundatom}{\ensuremath{\overline{p}}}
\newcommand{\groundatoms}{\ensuremath{\overline{\predicate}}}
\newcommand{\literal}{\ensuremath{l}}
\newcommand{\groundliteral}{\ensuremath{\overline{l}}}
\newcommand{\functions}{\ensuremath{\mathcal{F}}}
\newcommand{\function}{\ensuremath{F}}
\newcommand{\functionterm}{\ensuremath{f}}
\newcommand{\groundfunctionterm}{\ensuremath{\overline{f}}}
\newcommand{\groundfunctionterms}{\ensuremath{\overline{\function}}}
\newcommand{\functionexpression}{\ensuremath{e}}
\newcommand{\groundfunctionexpression}{\ensuremath{\overline{e}}}
\newcommand{\constraint}{\ensuremath{c}}
\newcommand{\groundconstraint}{\ensuremath{\overline{c}}}
\newcommand{\actionschemas}{\ensuremath{\mathcal{A}}}
\newcommand{\actionschema}{\ensuremath{a}}
\newcommand{\groundaction}{\ensuremath{\overline{a}}}
\newcommand{\objects}{\ensuremath{\mathcal{O}}}
\newcommand{\object}{\ensuremath{o}}
\newcommand{\variables}{\ensuremath{\mathcal{X}}}
\newcommand{\variable}{\ensuremath{x}}
\newcommand{\initial}{\ensuremath{\state_0}}
\newcommand{\goal}{\ensuremath{\mathcal{G}}}
\newcommand{\arity}[1]{\ensuremath{\mathit{ar}(#1)}}
\newcommand{\element}{\ensuremath{y}}
\newcommand{\groundelement}{\ensuremath{\overline{y}}}

\newcommand{\substitution}{\ensuremath{\rho}}
\newcommand{\substitutions}{\ensuremath{\mathcal{R}}}

\newcommand{\bigO}{\ensuremath{O}}

\newcommand{\precondition}[1]{\ensuremath{\mathit{pre}(#1)}}
\newcommand{\effect}[1]{\ensuremath{\mathit{eff}(#1)}}
\newcommand{\pprecondition}[1]{\ensuremath{\mathit{pre}^+(#1)}}
\newcommand{\mprecondition}[1]{\ensuremath{\mathit{pre}^-(#1)}}
\newcommand{\peffect}[1]{\ensuremath{\mathit{eff}^+(#1)}}
\newcommand{\meffect}[1]{\ensuremath{\mathit{eff}^-(#1)}}
\newcommand{\nprecondition}[1]{\ensuremath{\mathit{pre}^{\Delta}(#1)}}
\newcommand{\neffect}[1]{\ensuremath{\mathit{eff}^{\Delta}(#1)}}

\newcommand{\state}{\ensuremath{\mathcal{S}}}

\newcommand{\assignmentset}{\ensuremath{\lambda}}
\newcommand{\assignmentdegree}{\ensuremath{d}}
\newcommand{\consistencygraph}{\ensuremath{\mathcal{G}}}
\newcommand{\vertices}{\ensuremath{\mathcal{V}}}
\newcommand{\vertex}{\ensuremath{v}}
\newcommand{\edge}{\ensuremath{e}}
\newcommand{\edges}{\ensuremath{\mathcal{E}}}
\newcommand{\removaledges}{\ensuremath{\mathcal{I}}}
\newcommand{\clique}{\ensuremath{\mathcal{C}}}

\newcommand{\threesatformula}{\ensuremath{\phi}}
\newcommand{\threesatvariables}{\ensuremath{Y}}
\newcommand{\threesatvariable}{\ensuremath{y}}
\newcommand{\threesatliteral}{\ensuremath{l}}
\newcommand{\threesatclauses}{\ensuremath{Z}}
\newcommand{\threesatclause}{\ensuremath{z}}

\newcommand{\gadget}{\ensuremath{G}}

\begin{document}

\maketitle

\begin{abstract}

Most planners ground numeric planning tasks, given in a first-order-like language, into a ground task representation. However, this can lead to an exponential blowup in task representation size, which occurs in practice for hard-to-ground tasks. We extend a state-of-the-art lifted successor generator for classical planning to support numeric precondition applicability. The method enumerates maximum cliques in a substitution consistency graph. Each maximum clique represents a substitution for the variables of the action schema, yielding a ground action. We augment this graph with numeric action preconditions and prove the successor generator is exact under formally specified conditions. When the conditions fail, our generator may list inapplicable ground actions; a final applicability check filters these without affecting completeness. However, this cannot happen in 23 of 25 benchmark domains, and it occurs only in 1 domain. To the authors' knowledge, no other lifted successor generator supports numeric action preconditions. This enables future research on lifted planning for a very rich planning fragment.

\end{abstract}

\section{Introduction}

Numeric planning extends classical planning by introducing numeric state variables, constraints, and effects, thereby enabling the modeling of richer and more realistic domains \cite{fox-long-jair2003}. Although tasks are typically specified in a first-order-like representation, most planners fully ground them. This grounding step often performs well on many existing benchmarks but can lead to an exponential increase in representation size. At the same time, the ground representation facilitates the development of domain-independent heuristics to guide search \egcite{hoffmann-jair2003,aldinger-nebel-ki2017,scala-et-al-ijcai2017}, which are less straightforward to define directly over first-order representations. \emph{However,  exponentially larger ground tasks may not fit into memory, causing search to fail, thereby imposing a limitation on lifted search guidance}, such as lifted heuristics \egcite{chen-thiebaux-neurips2024,borelli-et-al-socs2025} or generalized policies \egcite{wang-thibaux-icaps2024} learned on smaller example tasks yet scaling to larger ones. Our work is also compatible with extensions of such to numeric planning \egcite{frances-et-al-aaai2021,correa-et-al-aaai2022,horcik-et-al-aaai2022,drexler-et-al-icaps2022,wichlacz-et-al-ijcai2022,stahlberg-et-al-kr2023}.

We address the \emph{lifted successor generation} problem. More precisely, given a first-order action schema and a state, the task is to enumerate exactly those ground actions that are applicable in the state. In comparison, grounded planners often instantiate ground actions independent of a specific state upfront, which can be exponentially many, and store them in a data structure that allows fast enumeration for a given state \cite{helmert-aij2009}. In the lifted successor generation setting, however, applicability must ensure consistent object-variable substitutions across all relevant lifted representations, such as literals and numeric constraints. To the best of our knowledge, existing lifted successor generators build upon conjunctive queries \cite{correa-et-al-icaps2020}, constraint satisfaction problem \cite{frances-phd2017}, or $k$-clique enumeration in $k$-partite graphs (KPKC) \cite{stahlberg-ecai2023}.

We build upon the KPKC-based approach, initially proposed for classical planning tasks, which formulates lifted successor generation as a $k$-clique enumeration problem in $k$-partite graphs. This approach encodes an action precondition as a $ k$-partite substitution consistency graph, where each $k$-clique represents an object-variable substitution of the action parameters yielding a candidate applicable ground action. The generator is sound and complete when predicates have arity at most two, meaning that it cannot generate inapplicable ground actions. However, in the general case, a verification of actual applicability is required.

We extend the KPKC-based approach to consider additional structural information from numeric planning tasks. Our approach incorporates numeric constraints into the substitution consistency graph. Consequently, we generalize the theoretical guarantees from the classical to the numeric planning fragment. Our theoretical result guarantees soundness and completeness if and only if literals, function terms, and constraints have arity at most two. We ran an empirical evaluation that validates this result, showing that inapplicable ground actions cannot occur in $23$ out of $25$ domains across two benchmark suites. Moreover, we observe such actions occurring in only one of the two candidate domains.

\section{Background}

In this section, we define the background and notation on interval arithmetic \cite{moore-et-al-2009} and numeric planning \cite{fox-long-jair2003,haslum-et-al-2019}. We consider extended real numbers $\extreals = \reals\cup\{-\infty,+\infty\}$.

\subsection{Interval Arithmetic}

A \defined{(closed) interval} is $\interval = [a,b] = \{r\in\extreals\mid a\leq r\leq b\}$ and the set of all intervals is $\intervals = \{[a,b]\mid a,b\in\extreals, a\leq b\}~\dot{\cup}~\{\emptyset\}$ where $\emptyset$ denotes the empty interval. 
For an arithmetic operator $\odot\in\{+,-,\times,\div\}$ and a comparison operator $\oplus\in\{=,<,>,\leq,\geq\}$, and for all $\interval,\interval'\in\intervals$, we define
\[
\begin{aligned}
\interval\odot \interval' &= \hull\bigl(\{x\odot y\mid x\in \interval, y\in \interval'~.~x\odot y \text{ is defined }\}\bigr), \\
\interval\oplus \interval' &= 
\begin{cases}
    \text{true} & \exists x\in \interval, y\in \interval'~.~x\oplus y \text{ evaluates to true,} \\
    \text{false} & \text{otherwise.}
\end{cases}
\end{aligned}
\]

where convex hull $\hull(S) = [\inf S, \sup S]$ for set $S$, with $\hull(\emptyset) = \emptyset$. 
If $S$ is unbounded from above (resp.~below) then $\sup S = +\infty$ (resp.~$\inf S = -\infty$).

\subsection{Numeric Planning}

A \defined{numeric planning task} is a tuple $\tup{\predicates,\functions,\actionschemas,\objects,\initial,\goal}$.
The \defined{planning domain} comprises $\tup{\predicates,\functions,\actionschemas}$, and the \defined{task information} comprises $\tup{\objects,\initial,\goal}$, namely:

\paragraph{Predicates and Literals.} $\predicates$ is a set of predicate symbols. 
Each \defined{predicate symbol} $\predicate$ in $\predicates$ has an associated arity $\arity{\predicate}$ in $\naturals$.
An \defined{atom} over a predicate $\predicate$ of arity $\arity{\predicate} = k$ is an expression $\predicate(\variable_1,\ldots,\variable_k)$ where $\variable_i$ with $i=1,\ldots,k$ are variables or constants.
An \defined{literal} is either an atom $\atom$ or its negation $\neg \atom$.

\paragraph{Functions and Numeric Expressions.} $\functions$ is a set of function symbols.
Each \defined{function symbol} $\function$ in $\functions$ has an associated arity $\arity{\function}$ in $\naturals$.
A \defined{function term} over a function symbol $\function$ of arity $\arity{\function} = k$ is an expression $\function(\variable_1,\ldots,\variable_k)$ where $\variable_i$ with $i=1,\ldots,k$ are variables or constants.
A \defined{function expression} is any expression built from function terms using arithmetic operators $\odot$.
A \defined{numeric constraint} is an expression of the form $\functionexpression\oplus\functionexpression'$ where $\functionexpression$ and $\functionexpression'$ are function expressions and $\oplus$ is a comparison operator.
A \defined{numeric effect} is an expression of the form $\functionterm \otimes \functionexpression$, where $\functionterm$ is a function term, 
$\functionexpression$ is a function expression, and $\otimes$ is one of $\{:=,\pluseq,\minuseq,\timeseq,\diveq\}$.

\paragraph{Objects and Substitutions.} $\objects$ is a set of objects (constants). 
A \emph{substitution function} $\substitution$ over some set of variables $\variables(\substitution)$ and objects $\objects$ is a function $\substitution : \variables(\substitution)\rightarrow\objects$ mapping each variable in $\variables(\substitution)$ to an object in $\objects$. We write $\substitutions(\variables,\objects)$ for the set of all $\substitution$ over $\variables$ and $\objects$. We write $\variable_1/\object_1,\ldots,\variable_n/\object_n$ for a substitution function $\substitution$ over variables $\variables = \{\variable_1,\ldots,\variable_n\}$ and objects $\objects$ such that $\substitution(\variable_i) = \object_i$ for all $i=1,\ldots,n$. We write $\substitution\subseteq\substitution'$ iff $\variables(\substitution)\subseteq\variables(\substitution')$ and $\substitution(\variable) = \substitution'(\variable)$ for all $\variable$ in $\variables(\substitution)$. The application of $\substitution$ to a syntactic element $\element$ (e.g., an atom, literal, function term, function expression, numeric constraint, numeric effect, or action), denoted by $\element[\substitution]$, replaces every occurrence of each variable $\variable_i$ with the corresponding object $\object_i = \substitution(\variable_i)$. The set of free variables occurring in an element $\element$ is denoted $\variables(\element)$, and the arity $\arity{\element}$ of $\element$ is set to $|\variables(\element)|$. We say that $\element$ is \defined{ground} and denote it by $\groundelement$ iff $\variables(\element)=\emptyset$.

\paragraph{States.} $\initial$ is the \defined{initial state}. A \defined{state} $\state$ consists of a set of ground atoms
$\groundatoms(\state)$ that hold in $\state$, and a set of ground function terms $\groundfunctionterms(\state)$.
For each $\groundfunctionterm\in\groundfunctionterms(\state)$, its value $\groundfunctionterm(\state)\in\reals$ is the value
associated with $\groundfunctionterm$ in $\state$. Ground atoms not in $\groundatoms(\state)$ are false; ground function terms
not in $\groundfunctionterms(\state)$ are undefined. A positive ground literal $\groundatom$ \defined{holds} in $\state$ iff
$\groundatom\in\groundatoms(\state)$, and a negative ground literal $\neg\groundatom$ \defined{holds} iff
$\groundatom\notin\groundatoms(\state)$. The value $\groundfunctionexpression(\state)$ (resp.\ $\groundconstraint(\state)$) of a
ground function expression (resp.\ ground numeric constraint) is defined in the usual way from the values
$\groundfunctionterm(\state)$. A ground numeric constraint \defined{holds} in $\state$ iff it evaluates to true.

\paragraph{Goals.} $\goal$ is the goal. The goal is a set of ground atoms and a set of ground numeric constraints.

\paragraph{Actions and Successors.} $\actionschemas$ is a set of actions. Each \defined{action} $\actionschema$ in $\actionschemas$ consists of two sets $\precondition{\actionschema}$ and $\effect{\actionschema}$ where $\precondition{\actionschema}$ is a set of precondition literals and numeric constraints and $\effect{\actionschema}$ is a set of effect literals and numeric effects. We write $\pprecondition{\actionschema}$,$\mprecondition{\actionschema}$,$\peffect{\actionschema}$, and $\meffect{\actionschema}$ for the sets of atoms of positive and negative literals in $\precondition{\actionschema}$ and $\effect{\actionschema}$, $\nprecondition{\actionschema}$ for the numeric constraints in $\precondition{\actionschema}$, and $\neffect{\actionschema}$ for the numeric effects in $\effect{\actionschema}$. A ground action $\groundaction$ is \defined{applicable} in a state $\state$ iff 
\begin{itemize}

\item every ground literal $\groundliteral$ and ground numeric constraint $\groundconstraint$ from $\precondition{\groundaction}$ is both defined and holds in $\state$,

\item for every ground numeric effect $\groundfunctionterm\otimes \groundfunctionexpression\in\effect{\groundaction}$, $\groundfunctionexpression(\state)$ is defined; and if $\otimes\neq :=$, then $\groundfunctionterm(\state)$ is also defined, and

\item the set of ground numeric effects in $\groundaction$ is \defined{non-conflicting}, i.e., there is at most one effect $\groundfunctionterm\otimes \groundfunctionexpression$ for a ground function term $\groundfunctionterm$ or
all effects $\groundfunctionterm\otimes \groundfunctionexpression$ that share the same $\groundfunctionterm$ are unambiguous, i.e., $\otimes\in\{\pluseq,\minuseq\}$ or $\otimes\in\{\timeseq,\diveq\}$.

\end{itemize}

Applying a ground action $\groundaction$ in a state $\state$ where it is applicable yields the \defined{successor state} $\state'$ with 
\begin{itemize}

\item $\groundatoms(\state') = (\groundatoms(\state)\setminus\meffect{\groundaction})\cup\peffect{\groundaction}$

\item $\groundfunctionterm(\state') = \groundfunctionterm(\state)$ if $\not\exists \groundfunctionterm\otimes\groundfunctionexpression\in\neffect{\groundaction}$  

\item $\groundfunctionterm(\state') = \groundfunctionexpression(\state)$ if $\groundfunctionterm :=\groundfunctionexpression\in\neffect{\groundaction}$

\item $\groundfunctionterm(\state') = \groundfunctionterm(\state) + \sum\limits_{\groundfunctionterm\pluseq\groundfunctionexpression\in\neffect{\groundaction}} \groundfunctionexpression(\state) - \sum\limits_{\groundfunctionterm\minuseq\groundfunctionexpression\in\neffect{\groundaction}} \groundfunctionexpression(\state)$ 

if $\exists \groundfunctionterm\pluseq\groundfunctionexpression\in\neffect{\groundaction}$ or $\exists \groundfunctionterm\minuseq\groundfunctionexpression\in\neffect{\groundaction}$

\item $\groundfunctionterm(\state') = \groundfunctionterm(\state) \cdot \prod\limits_{\groundfunctionterm\timeseq\groundfunctionexpression\in\neffect{\groundaction}} \groundfunctionexpression(\state) \cdot \prod\limits_{\groundfunctionterm\diveq\groundfunctionexpression\in\neffect{\groundaction}} \frac{1}{\groundfunctionexpression(\state)}$ 

if $\exists \groundfunctionterm\timeseq\groundfunctionexpression\in\neffect{\groundaction}$ or $\exists \groundfunctionterm\diveq\groundfunctionexpression\in\neffect{\groundaction}$

\end{itemize}

The objective for a given numeric planning task is to find a \defined{plan}, which is a sequence of ground actions, that when successively applied starting from the initial state, yields a state where the goal holds, i.e., all ground literals and ground numeric constraints mentioned in the goal hold.

\subsection{Substitution Consistency Graph}

The \defined{substitution consistency graph} $\consistencygraph_{\actionschema,\state}$ is a tuple $\tup{\vertices, \edges}$ where 
$\vertices = \{\variable/\object \mid \variable\in\variables(\actionschema), \object\in\objects \}$ is the set of vertices, each representing a possible variable substitution, and 
$\edges = \vertices^2\setminus (\removaledges^{\neq}\cup \removaledges^+\cup \removaledges^-)$ is the set of edges, each representing all possible pairs of variable substitutions \cite{stahlberg-ecai2023}.
An atom $\predicate(x_1,\ldots,x_n)$ \defined{matches} an atom $\predicate'(x_1',\ldots,x_n')$ iff $\predicate = \predicate'$ 
and for all $i=1,\ldots,n$, $x_i$ or $x_i'$ are variables, or else $x_i = x_i'$, i.e., both are objects.
The sets $\removaledges^{\neq}, \removaledges^+, \removaledges^-$ are:

\begin{itemize}

\item $\removaledges^{\neq} = \{\{x/o_1,x/o_2\} \mid x\in\variables(\actionschema), o_1,o_2\in\objects\}$, 
i.e., the set of edges whose corresponding substitution function assigns different objects to the same variable,

\item $\removaledges^+ = \{\{\vertex,\vertex'\}\mid \vertex,\vertex'\in \vertices,\exists \atom\in\pprecondition{\actionschema},\forall \atom'\in\state~.~\atom[\vertex,\vertex'] \text{ does not match } \atom'\}$,
i.e., the set of edges whose corresponding substitution function, when applied to any atom of a positive literal in the precondition, does not match any ground atom in the state.

\item $\removaledges^- = \{\{\vertex,\vertex'\}\mid \vertex,\vertex'\in \vertices,\exists \atom\in\mprecondition{\actionschema}~.~\groundatom[\vertex,\vertex']\in\state\}$, 
i.e., the set of edges whose corresponding substitution function, when applied to any atom of a negative literal in the precondition, always yields a ground atom that is part of the state.

\end{itemize}

\Omit{
\begin{lemma}

Consider an action $\actionschema$ with arity $k\geq 2$, a state $\state$, and a $k$-clique $\clique = \{\vertex_1,\ldots,\vertex_k\}$ in $\consistencygraph_{\actionschema,\state}$. Then, $\groundaction[\vertex_1,\ldots,\vertex_k]$, i.e., $\actionschema[\vertex_1,\ldots,\vertex_k]$ is a ground action $\groundaction$.

\label{lemma:classical:ground}

\end{lemma}
}

From \inlinecite{stahlberg-ecai2023}, it follows the following key theorem that allows using k-clique enumeration algorithms to generate precisely all applicable ground actions in a state.

\begin{theorem}

Consider an action $\actionschema$ with arity $k\geq 2$ where each literal $\literal\in\precondition{\actionschema}$ has an arity of at most $2$, and a state $\state$. Then, there exists a $k$-clique $\clique = \{\vertex_1,\ldots,\vertex_k\}$ in $\consistencygraph_{\actionschema,\state}$ iff all ground literals $\groundliteral = \literal[\vertex_1,\ldots,\vertex_k]$ with $\literal\in\precondition{\actionschema}$ hold in $\state$.

\label{thm:classical:sound_and_complete}

\end{theorem}

Theorem~\ref{thm:classical:sound_and_complete} does not apply to numeric planning because the substitution consistency graph ignores numeric constraints in the action precondition. Nevertheless, we can use the substitution consistency graph to generate ground applicable actions for numeric planning tasks. However, this typically leads to the generation of inapplicable actions. Therefore, in the following section, we will integrate numeric constraints into the graph while generalizing Theorem~\ref{thm:classical:sound_and_complete}.
\section{Numeric Successor Generation}

In this section, we integrate numeric constraints into the edge exclusion criterion of the consistency graph. We say that a numeric constraint $\constraint$ with nonempty free variables $\variables(\constraint)$ is \defined{satisfiable} in a state $\state$, written $\state\vDash \constraint$ iff there exists a substitution function $\substitution$ in $\substitutions(\variables(\constraint),\objects)$ such that the ground numeric constraint $\constraint[\substitution]$ is satisfiable in $\state$, and otherwise it is \defined{unsatisfiable}, written $\state\nvDash \constraint$. Finally, we add an edge removal set $\removaledges^{\Delta}$ based on numeric constraints,
resulting in the set of edges $\edges = \vertices^2\setminus (\removaledges^{\neq}\cup \removaledges^+\cup \removaledges^-\cup \removaledges^{\Delta})$ with

\begin{itemize}
    \item $\removaledges^{\Delta} = \{\{\vertex,\vertex'\} \mid \vertex,\vertex'\in \vertices,\exists \constraint\in\nprecondition{\actionschema}~.~\state\nvDash \constraint[\vertex,\vertex'] \}$, 
    i.e., unsatisfiable numeric constraints under any substitution of the remaining variables.
\end{itemize}

\subsection{Soundness and Completeness}

The following theorem generalizes the soundness and completeness results from Theorem~\ref{thm:classical:sound_and_complete} toward the extended consistency graph that considers numeric constraints.

\begin{restatable}{theorem}{theoremnumeric}
    Consider an action $\actionschema$ with arity $k\geq 2$ where 
    each numeric constraint $\constraint$ in $\precondition{\actionschema}$ has an arity of at most $2$, 
    and a state $\state$. 
    Then, there exists a $k$-clique $\clique = \{\vertex_1,\ldots,\vertex_k\}$ in $\consistencygraph_{\actionschema,\state}$
    iff all ground numeric constraints $\groundconstraint = \constraint[\vertex_1,\ldots,\vertex_k]$ with $\constraint$ in $\precondition{\actionschema}$ hold in $\state$.
    \label{thm:numeric:sound_and_complete}
\end{restatable}

We can now use the consistency graph to generate bindings that yield ground actions whose precondition components hold in a given state. \emph{Note that such actions may still be inapplicable due to problematic ground effects (see definition of applicability). However, integrating such is analogous to imposing numeric constraints.} Instead, we focus on an efficient method for underapproximating constraint unsatisfiability in the edge removal criterion of $\removaledges^\Delta$ because this problem is coNP-complete (proof in appendix).

\subsection{Leaf-Localized Interval Relaxation}

To efficiently underapproximate constraint unsatisfiability in the edge removal criterion of $\removaledges^\Delta$, we adopt a leaf-localized interval relaxation. The relaxation recursively evaluates a numeric constraint in three stages: (leaf) partially evaluates each function term under the edge substitution, (inner) combines the resulting intervals using interval arithmetic to evaluate function expressions, and (root) applies existential interval comparison to decide satisfiability of the constraint. At the leaf level, we generalize assignment sets \cite{stahlberg-ecai2023} from Boolean to interval outputs. A (numeric) \defined{assignment set} for a function symbol $\function$ and generic function term $\functionterm = \function(\variable_1,\ldots,\variable_k)$ with variables $\variables = {\variable_1,\ldots,\variable_k}$ in a state $\state$ is a function $\assignmentset_{\function,\state} : \bigcup_{\variables'\subseteq \variables}\substitutions(\variables',\objects)\rightarrow\intervals\cup\{\emptyset\}$ with 
\[\assignmentset_{\function,\state}(\substitution) = \hull\bigl(\{\functionterm[\substitution'](\state)\mid\substitution'\in\substitutions(\variables,\objects), \substitution\subseteq\substitution'\}\bigr)\]

The following proposition shows that we can efficiently compute an assignment set, assuming that we restrict the number of substitutions $\assignmentdegree$ in its input domain.

\begin{restatable}{proposition}{propositionassignmentset}

Consider a state $\state$ with $n$ ground function terms over a function symbol $\function$ with arity $\arity{\function} = k$ and variables $\variables = \{x_1,\ldots,x_k\}$, and an integer $\assignmentdegree$ in $\mathbb{N}_0$.
We can construct an assignment set $\assignmentset_{\function,\state}$ for $\function$ in $\state$ with domain $\bigcup_{\variables'\subseteq \variables,|\variables'|\leq \assignmentdegree}\substitutions(\variables',\objects)$ in $\bigO(n\cdot k^\assignmentdegree |\objects|^\assignmentdegree)$ time, $\bigO(n + k^\assignmentdegree |\objects|^\assignmentdegree)$ space. 

\end{restatable}

Intuitively, $\assignmentset_{\function,\state}(\substitution)$ captures the range of values that the function term can take after fixing some variables (an overapproximation). More precisely, we first apply the partial substitution $\substitution$ to $\functionterm$. Then, we assign arbitrary objects to the remaining variables, yielding all ground function terms in $\state$, and take the smallest interval containing all of their values. In our setting, each such $\substitution$ corresponds to an edge $\{\vertex,\vertex'\}$ under consideration, so $\assignmentdegree = 2$ suffices. Interval arithmetic preserves this overapproximation when evaluating function expressions. Finally, the existential interval comparison also overapproximates satisfiability: it may miss some unsatisfiable constraints but never declares a satisfiable one as unsatisfiable. The following corollary summarizes the key result.

\begin{restatable}{corollary}{corollaryrelaxation}

The result of Theorem~\ref{thm:numeric:sound_and_complete} also holds when applying the interval-based relaxation using assignment sets, provided that each function in $\functions$ has arity at most two.

\end{restatable}

\newcommand{\numtasks}[1]{\small{(#1)}}
\setlength{\tabcolsep}{3pt}

\begin{table*}[ht]
  \centering
  \caption{
      Comparison of successor generators. For each domain, we report the number of solved tasks (S), the total time in seconds for commonly solved tasks (T), and, for lifted configurations, the overapproximation ratio (OA). Best values (lowest time or OA = 1.00) among lifted configurations are shown in bold.}
  \label{tab:results}
  \begin{tabular}{l l rr rrr rrr rrr}
    \toprule
    &                                            & \multicolumn{5}{c}{Baselines}                                  & \multicolumn{6}{c}{KPKC (Substitution Consistency Graph)}                      \\
                                                   \cmidrule(lr){3-7}                                               \cmidrule(l){8-13}
    & Domain                    & \multicolumn{2}{c}{Grounded}  & \multicolumn{3}{c}{Exhaustive} & \multicolumn{3}{c}{Propositional} & \multicolumn{3}{c}{Numeric} \\
                                                   \cmidrule(lr){3-4}              \cmidrule(lr){5-7}               \cmidrule(lr){8-10}                  \cmidrule(l){11-13}
    &                                            & S   & T                    & S   & OA      & T           & S        & OA  & T                  & S        & OA  & T              \\
    \midrule
    {\multirow{18}{*}{\rotatebox[origin=c]{90}{IPC}}}
    % & block-grouping \numtasks{20}             & --  & --                      & --  & --      & --             & --       & --        & --                   & --       & --             & --                \\
    & counters \numtasks{20}                     & 3   & 3                       & 3   & 1.16    & 3              & 3        & 1.16      & \bf{3}               & 3        & \bf{1.00}      & \bf{3}            \\
    & delivery \numtasks{20}                     & 2   & 330                     & 2   & 34.58   & 681            & 2        & \bf{1.00} & 636                  & 2        & \bf{1.00}      & \bf{632}          \\
    & drone \numtasks{20}                        & 4   & --                      & --  & --      & --             & 4        & --        & --                   & 4        & --             & --                \\
    & expedition \numtasks{20}                   & 6   & 361                     & 6   & 63.44   & 1\,034         & 6        & 1.76      & 724                  & 6        & \bf{1.00}      & \bf{584}          \\
    % & ext-plant-watering \numtasks{20}         & --  & ---                     & --  & --      & --             & --       & --        & --                   & --       & --             & --                \\
    & farmland \numtasks{20}                     & 5   & 323                     & 4   & 2.05    & \bf{368}       & 4        & 1.02      & 543                  & 4        & \bf{1.00}      & 550               \\
    & fo-counters \numtasks{20}                  & 4   & 21                      & 4   & 1.27    & \bf{20}        & 4        & 1.27      & 26                   & 4        & \bf{1.00}      & 25                \\
    & fo-farmland \numtasks{20}                  & 4   & --                      & --  & --      & --             & 4        & --        & --                   & 4        & --             & --                \\
    & fo-sailing \numtasks{20}                   & 3   & 373                     & 3   & 1.32    & \bf{507}       & 3        & 1.32      & 609                  & 3        & \bf{1.00}      & 554               \\
    & hydropower \numtasks{20}                   & 12  & 44                      & 8   & 1851.70 & 2\,171         & 10       & 1.25      & 258                  & 10       & \bf{1.00}      & \bf{232}          \\
    % & markettrader \numtasks{20}               & --  & --                      & --  & --      & --             & --       & --        & --                   & --       & --             & --                \\
    & mprime \numtasks{20}                       & 5   & 670                     & 5   & 5.21    & \bf{805}       & 5        & 1.11      & 1\,110               & 5        & \bf{1.00}      & 929               \\
    & pathwaysmetric \numtasks{20}               & 1   & 727                     & 1   & 19.43   & 1\,352         & 1        & 1.75      & 1\,223               & 1        & \bf{1.00}      & \bf{887}          \\
    & rover \numtasks{20}                        & 4   & 65                      & 4   & 35.24   & 35             & 4        & 1.90      & 35                   & 4        & 1.01           & \bf{27}           \\
    % & sailing \numtasks{20}                    & --  & --                      & --  & --      & --             & --       & --        & --                   & --       & --             & --                \\
    % & settlersnumeric \numtasks{20}            & --  & --                      & --  & --      & --             & --       & --        & --                   & --       & --             & --                \\
    & sugar \numtasks{20}                        & 2   & 224                     & 2   & 21.34   & \bf{328}       & 2        & 2.32      & 433                  & 2        & \bf{1.00}      & 367               \\
    & tpp \numtasks{20}                          & 2   & 2                       & 2   & 10.66   & \bf{2}         & 2        & 1.55      & \bf{2}               & 2        & \bf{1.00}      & \bf{2}            \\
    & \textit{others; entirely unsolved} \numtasks{100}             & 0  & --                       & 0   & --      & --             & 0        & --        &                      & 0        & --             & --                \\           
    \midrule
    & Total \numtasks{380}                       & 57  & 3\,143                  & 44  & 12.38   & 7\,306         & \bf{54}  & 1.40      & 5\,602               & \bf{54}  & \bf{1.00} & \bf{4\,791}            \\
    \midrule
    {\multirow{8}{*}{\rotatebox[origin=c]{90}{MinePDDL}}}
    % & bridge \numtasks{3} \\
    % & build-cabin \numtasks{3} \\
    % & build-cross \numtasks{3} \\
    % & build-shape \numtasks{3} \\
    % & build-wall \numtasks{3} \\
    % & build-well \numtasks{3} \\
    % & climb-place \numtasks{3} \\
    % & collect-build-shape \numtasks{3} \\
    % & cut-tree \numtasks{3} \\
    & gather-multiple-wood \numtasks{3}          & 1   & 20                      & 1   & 75\,145.52  & 1\,430             & 1        & 3\,383.96 & 251                   & 1        & \bf{1.00}      & \bf{64}           \\
    & gather-wood \numtasks{3}                   & 2   & 241                     & 2   & 263\,215.73 & 100                & 2        & 6\,430.66 & 8                     & 2        & \bf{1.00}      & \bf{6}            \\
    & move-to-loc \numtasks{3}                   & 1   & 13                      & 1   & 117\,773.24 & 16                 & 1        & 3\,046.39 & \bf{2}                & 1        & \bf{1.00}      & 1                 \\
    & pickup-and-place \numtasks{3}              & 1   & --                      & --  & --          & --                 & 1        & --        & --                    & 1        & --             & --                \\
    & pickup-diamond \numtasks{3}                & 2   & 13                      & 1   & 118\,539.73 & 21                 & 2        & 3\,658.47 & 3                     & 2        & \bf{1.00}      & \bf{2}            \\
    & place-wood \numtasks{3}                    & 1   & 12                      & 1   & 77\,987.83  & 90                 & 1        & 4\,365.36 & 14                    & 1        & \bf{1.00}      & \bf{3}            \\
    & scaled-move-to-loc \numtasks{150}          & 150 & 2\,255                  & 145 & 133\,756.08 & 16\,815            & 150      & 3\,580.43 & 579                   & 150      & \bf{1.00}      & \bf{386}          \\
    & \textit{others; entirely unsolved} \numtasks{27}              & 0   & --                      & 0   & --          & --                 & 0        & --        & --                    & 0        & --             & --                \\           
    \midrule
    & Total \numtasks{195}                       & 158 & 2\,554                  & 151 & 119\,282.75 & 18\,472            & \bf{158} & 3\,948.69 & 857                   & \bf{158} & \bf{1.00}      & \bf{462}          \\
    \bottomrule
  \end{tabular}
\end{table*}

\section{Experiments}

We evaluated our successor generator on two benchmark sets: (1) all numeric domains from the 2023 IPC and (2) all numeric domains from the MinePDDL benchmark suite \cite{hill-et-al-icaps2024}, which are challenging for both lifted and grounded numeric planners. We consider two baseline configurations: \defined{Grounded}, which constructs a ground task and generates a decision tree using CART \cite{breiman-et-al-1984} with Gini-based node selection over precondition frequencies, and \defined{Exhaustive}, which enumerates all candidate bindings consistent with type annotations. For the $k$-clique enumeration in the substitution consistency graph, we compare \defined{Propositional}, which ignores $\removaledges^\Delta$, against \defined{Numeric}, which includes it. We ran each configuration using $A^*$ search \cite{hart-et-al-ieeessc1968} with a blind heuristic, a 30-minute time limit, and an 8 GB memory limit under the Lab toolkit \cite{seipp-et-al-zenodo2017}.

Table~\ref{tab:results} shows that \textit{Exhaustive} yields the lowest coverage, highest runtime, and largest overapproximation ratio. As more edges are excluded, i.e., moving toward \textit{Propositional} and further to \textit{Numeric}, we observe the following trends: coverage remains stable or improves, runtime generally decreases, and the overapproximation ratio either remains constant or decreases, often drastically on MinePDDL. Most notably, the overapproximation ratio for \textit{Numeric} is 1.00 in all but one domain, specifically the rover domain. Theoretically, soundness conflicts can occur only in rovers and pathwaysmetric, due to predicates, and in the latter, also functions/constraints of arity three.

\section{Conclusions}

We presented the first lifted successor generator for numeric planning, circumventing the worst-case exponential blowup of instantiating a ground task representation while enabling search directly over a first-order representation of numeric planning tasks. We have demonstrated that the key properties of its predecessor, initially developed for the classical planning fragment, naturally carry over to the more expressive numeric setting. Finally, we have empirically validated our theoretical results by demonstrating that it generates precisely all ground actions with applicable precondition in all but one benchmark domain, while offering computational benefits by avoiding the grounding phase.
\section{Appendix}

\appendix

\theoremnumeric*

\begin{proof}

We prove both directions by contradiction. 

``$\Rightarrow$ (Soundness)'': Assume that there exists a $k$-clique $\clique = \{\vertex_1,\ldots,\vertex_k\}$ in the graph $\consistencygraph_{\actionschema,\state}$ and all ground constraint $\constraint[\vertex_1,\ldots,\vertex_k]$ with $\constraint$ in $\precondition{\groundaction}$ do not hold in $\state$. Let the highest arity of a constraint be $\arity{\constraint} = 1$, and $x/o\in \clique$ be the vertex corresponding to the substitution that grounds $\constraint$. Since $k\geq 2$, there must be an $\edge\in\edges$ with $x/o\in\edge$. However, $\state\nvDash \constraint[\vertex_1,\ldots,\vertex_k]$, hence, $\edge\in \removaledges^\Delta$, and thus, $\edge\notin\edges$. Therefore, $x/o$ cannot be part of a clique of size $k\geq 2$. We conclude, $\edge\in \edges$ and $\edge\notin\edges$. This is a contradiction. Analogously, for the case where $\arity{\constraint} = 2$.

``$\Leftarrow$ (Completeness)'': Assume that all ground numeric constraints $\constraint[\vertex_1,\ldots,\vertex_k]$ with $\constraint$ in $\precondition{\groundaction}$ hold in $\state$ and there does not exist a clique $\clique = \{\vertex_1,\ldots,\vertex_k\}$ in the graph $\consistencygraph_{\actionschema,\state}$. Since $\state \vDash \constraint[\vertex_1,\ldots,\vertex_k]$ for all $\constraint$ in $\precondition{\groundaction}$, each constraint does not contribute to $\removaledges^\Delta$ and therefore cannot remove any edges from $\edges$. Thus, no edge was removed, and there must be a clique $\clique = \{\vertex_1,\ldots,\vertex_k\}$. This contradicts the assumption that no such clique exists.

\end{proof}

\propositionassignmentset*

\begin{proof} 

The proof is constructive. Algorithm~\ref{alg:assignmentset} defines a procedure to construct the assignment set $\lambda_\function$ for a function symbol $\function$ in a state $\state$ with a maximal substitution degree of $\assignmentdegree$. 
Lines~\ref{line:init:start}--\ref{line:init:end} initialize $\lambda_\function$ to $\emptyset$ for each substitution function $\substitution$ such that $|\variables(\substitution)| \leq \assignmentdegree$. 
There are $\bigO(k^\assignmentdegree |\objects|^\assignmentdegree)$ such entries, each requiring $\bigO(1)$ space to store an interval or $\emptyset$. 
Lines~\ref{line:set:start}--\ref{line:set:end} iterate over each ground function term $\groundfunctionterm = \function(\object_1,\ldots,\object_n)$ in $\state$. For each such term, the nested loop in lines~\ref{line:set:start2}--\ref{line:set:end2} iterates over every substitution function $\substitution$ such that $\function(\variable_1,\ldots,\variable_{\arity{\function}})[\substitution]$ matches $\groundfunctionterm$, where matching is defined analogous to atoms. 
This inner loop can require at most $\bigO(k^\assignmentdegree |\objects|^\assignmentdegree)$ iterations per ground function term with $\bigO(1)$ time per iteration, resulting in an overall complexity of $\bigO(n \cdot k^\assignmentdegree |\objects|^\assignmentdegree)$ time and $\bigO(n + k^\assignmentdegree |\objects|^\assignmentdegree)$ space.

\begin{algorithm}[h]
\caption{Construct an assignment set $\lambda_\function$ with maximal number of substitutions equal to $\assignmentdegree$.}
\label{alg:assignmentset}
\begin{algorithmic}[1]
    \STATE \textbf{Function} \textsc{Construct}$(\function, \state, \assignmentdegree)$
    \STATE $n\gets\arity{\function}$ \label{line:init:start}
    \STATE $\assignmentdegree\gets\min\{\assignmentdegree, n\}$
    \STATE $\variables\gets \{\variable_1,\ldots,\variable_n\}$
    \FOR{$j = 0,\ldots,\assignmentdegree$}
    \FOR{$\substitution\in\substitutions(\variables,\objects)$ with $|\variables(\substitution)| = j$ }
    \STATE $\lambda_{\function,\state}(\substitution)\gets \emptyset$
    \ENDFOR
    \ENDFOR \label{line:init:end}
    \FOR{$\groundfunctionterm = \function(\object_1,\ldots,\object_n)\in \groundfunctionterms(\state)$} \label{line:set:start}
    \FOR{$k = 0,\ldots,\assignmentdegree$} \label{line:set:start2}
    \FOR{$\{j_1,\ldots,j_k\} \in \binom{\{1,\ldots,n\}}{k}$}
    \STATE $\substitution\gets\variable_{j_1}/\object_{j_1},\ldots,\variable_{j_k}/\object_{j_k}$
    \STATE $\lambda_{\function,\state}(\substitution)$$\gets \hull\bigl(\lambda_{\function,\state}(\substitution) \cup \{\groundfunctionterm(\state)\}\bigr)$
    \ENDFOR
    \ENDFOR \label{line:set:end2}
    \ENDFOR \label{line:set:end}
    \STATE \textbf{return} $\lambda_{\function,\state}$
    \STATE
\end{algorithmic}
\end{algorithm}

\end{proof}

\corollaryrelaxation*

\begin{proof}

When constructing an assignment set, the generic function term is defined as $\function(x_1,\ldots,x_k)$ with distinct variables in each argument position. However, actual function terms in constraints may reuse variables, e.g., $\function(x,\ldots,x)$. In this case, assignment sets implicitly treat different argument positions as if they were independent, even though in a specific function term, they refer to the same variable. When every function has arity at most two, each function term can involve at most two variables, possibly with repetition. In this case, an assignment set with $\assignmentdegree=2$ still enumerates all relevant substitutions. Hence, partial evaluation coincides with exact evaluation, the edge removal set $\removaledges^\Delta$ remains unchanged, and the proof of Theorem~\ref{thm:numeric:sound_and_complete} carries over.

\end{proof}

\begin{theorem}
 The numeric constraint satisfiability problem (NC-SAT) is NP-complete.
\end{theorem}

\begin{proof}

First, we show membership. Given a numeric constraint $\constraint$ with free variables $\variables(\constraint)$ and a state $\state$, we can nondeterministically guess a substitution $\substitution \in \substitutions(\variables(\constraint), \objects)$ in polynomial time, since both $\variables(\constraint)$ and $\objects$ are finite. Evaluating the ground constraint $\constraint[\substitution](\state)$ is also polynomial, as it requires computing arithmetic expressions followed by a single comparison.

Second, we prove hardness by reduction from 3-SAT. We construct, in polynomial time, a numeric constraint $\constraint$ from a 3-SAT formula $\threesatformula$ that is satisfiable in a state $\state$ if and only if $\threesatformula$ is satisfiable.

Consider a propositional formula $\threesatformula$ over variables $\threesatvariables = \{\threesatvariable_1,\ldots,\threesatvariable_n\}$ with clauses $\threesatclauses = \{\threesatclause_1,\ldots,\threesatclause_m\}$ where each clause consists of exactly three literals. For the numeric constraint $\constraint$ and state $\state$, we define the set of objects $\objects = \{\object_\top,\object_\bot\}$, the set of free variables $\variables = \{\variable_\threesatvariable\mid \threesatvariable\in\threesatvariables\}$, and the set of unary function symbols $\functions = \{\function_\threesatvariable\mid\threesatvariable\in\threesatvariables\}$. Moreover, we define for each $\threesatvariable$ in $\threesatvariables$, the following ground functions terms with associated values $\function_\threesatvariable(\object_\top)(\state) = 1$ and $\function_\threesatvariable(\object_\bot)(\state) = 0$. We now encode the semantics of the formula $\threesatformula$ in the constraint $\constraint$. The construction proceeds by defining numeric gadgets for each component of $\phi$. 

\begin{itemize}

\item \textbf{Literal gadget.} For each literal $l$ over a variable $\threesatvariable$ in $\threesatvariables$, we define a literal gadget $\gadget_\threesatliteral$ that defines a ground function term as

$\gadget_\threesatliteral :=
\begin{cases}
 \function_\threesatvariable(x_\threesatvariable) & \text{ if } \threesatliteral = \threesatvariable \\
    1 - \function_\threesatvariable(x_\threesatvariable) & \text{ if } \threesatliteral = \neg \threesatvariable
\end{cases}
$

\item \textbf{Clause gadget.} For each clause $\threesatclause = (\threesatliteral_1\lor\ldots\lor \threesatliteral_3)$, define the function expression 
\[\gadget_\threesatclause := 1 - \prod\limits_{i=1}^3 (1 - \gadget_{\threesatliteral_{i}})\] 

\item \textbf{Formula gadget.} For the formula $\threesatformula$, we define the numeric constraint gadget $\gadget_{\threesatformula}$

\[\gadget_\threesatformula := \left(\sum\limits_{j=1}^{m}\gadget_{\threesatclause_j}\right) = m\]

\end{itemize}

Finally, we show that $\constraint$ is satisfiable in a state $\state$ if and only if the 3-SAT formula $\threesatformula$ is satisfiable.

``$\Leftarrow$'': Consider a truth assignment to the variables $\threesatvariables$ that satisfies the formula $\threesatformula$. Consider the substitution function $\theta$ with $\theta(\variable_\threesatvariable) = \object_\top$ (resp~$\object_\bot$) if $\threesatvariable$ is true (resp.~false) in the truth assignment. Next, we show that $\constraint[\theta](\state)$ is true. For any satisfied literal $\threesatliteral = \threesatvariable$ (resp.~$\threesatliteral = \neg\threesatvariable$), we have $\gadget_\threesatliteral = \function_\threesatvariable(x_\threesatvariable)[x_\threesatvariable/\object_\top](\state) = 1$ (resp.~$\gadget_\threesatliteral = 1 - \function_\threesatvariable(x_\threesatvariable)[x_\threesatvariable/\object_\bot](\state) = 1 - 0 = 1$). Moreover, for an unsatisfied literal $\threesatliteral = \threesatvariable$ (resp.~$\threesatliteral = \neg\threesatvariable$), we have $\gadget_\threesatliteral = \function_\threesatvariable(x_\threesatvariable)[x_\threesatvariable/\object_\bot] = 0$ (resp.~$1 - \function_\threesatvariable(x_\threesatvariable)[x_\threesatvariable/\object_\top] = 1 - 1 = 0$). More generally speaking, it holds that $\gadget_\threesatliteral = 1$ if and only if $\threesatliteral$ is satisfied, and otherwise $\gadget_\threesatliteral = 0$. Consequently, for each clause $\threesatclause$, we have $\gadget_\threesatclause = 1$ because at least one literal in each clause holds, and hence at least one product term $(1 - \gadget_\threesatliteral) = 0$ such that the whole product yields $0$ and $1 - 0 = 1$. Finally, $\gadget_\threesatformula = m$, showing that the constraint is satisfied in $\state$.

$\Rightarrow$ Consider a substitution function $\theta$ that satisfies the constraint $\constraint$ in $\state$. Observe that $\gadget_\threesatliteral$ is either $0$ or $1$ because $\function_\threesatvariable(x_\threesatvariable)[x_\threesatvariable/\object_\top](\state) = 1$, $\function_\threesatvariable(x_\threesatvariable)[x_\threesatvariable/\object_\bot] = 0$, and subtracting $1$ or $0$ from $1$ yields $0$ or $1$. Consequently, $\gadget_\threesatclause = 1$ or $\gadget_\threesatclause = 0$ because the product of combinations of numbers of $0$ and $1$ is either $0$ or $1$, and similar subtraction argument. Hence, since the constraint is satisfied, every $\gadget_\threesatclause$ must be $1$ such that the sum over all clauses yields $m$. Consequently, the product in each $\gadget_\threesatclause$ must be $0$, which requires at least one of its $\gadget_\threesatliteral = 1$. By construction, $\gadget_\threesatvariable = 1$ and $\gadget_{\neg\threesatvariable} = 0$ if $\theta(\variable_\threesatvariable) = \object_\top$. Moreover, $\gadget_\threesatvariable = 0$ and $\gadget_{\neg\threesatvariable} = 1$ if $\theta(\variable_\threesatvariable) = \object_\bot$. Consequently, we can directly retrieve the truth assignment to the variable $\threesatvariable$, i.e., $\top$ if $\theta(\variable_\threesatvariable) = \object_\top$ or $\bot$ if $\theta(\variable_\threesatvariable) = \object_\bot$.

\end{proof}

\begin{theorem}
 The numeric constraint unsatisfiability problem (NC-UNSAT) is coNP-complete.
\end{theorem}

\begin{proof}
 It follows immediately that NC-UNSAT is coNP-complete, since it is the complement of NC-SAT, which is NP-complete.
\end{proof}

\bibliography{abbrv,literatur,extra-literatur,crossref,extra-crossref}

\end{document}